\documentclass[11pt]{article}

\usepackage[utf8]{inputenc}
\usepackage[T1]{fontenc}
\usepackage{lmodern}
\usepackage[margin=1in]{geometry}
\usepackage{amsmath,amssymb,amsthm,mathtools}
\usepackage{bm}
\usepackage{graphicx}
\usepackage{booktabs}
\usepackage{enumitem}
\usepackage{microtype}
\usepackage[round,authoryear]{natbib}
\usepackage[colorlinks=true,linkcolor=blue!50!black,citecolor=blue!50!black,urlcolor=blue!50!black]{hyperref}
\usepackage[capitalise]{cleveref}
\usepackage{xcolor}
\crefname{problem}{Open Problem}{Open Problems}
\Crefname{problem}{Open Problem}{Open Problems}
\crefname{conjecture}{Conjecture}{Conjectures}
\Crefname{conjecture}{Conjecture}{Conjectures}

\theoremstyle{plain}
\newtheorem{theorem}{Theorem}[section]

\newtheorem{proposition}[theorem]{Proposition}
\newtheorem{corollary}[theorem]{Corollary}
\theoremstyle{definition}
\newtheorem{definition}[theorem]{Definition}
\newtheorem{example}[theorem]{Example}
\newtheorem{remark}[theorem]{Remark}
\newtheorem{conjecture}[theorem]{Conjecture}
\newtheorem{problem}[theorem]{Open Problem}

\newcommand{\R}{\mathbb{R}}
\newcommand{\N}{\mathbb{N}}
\newcommand{\E}{\mathbb{E}}
\newcommand{\Prob}{\mathbb{P}}
\newcommand{\W}{\mathcal{W}}
\newcommand{\Wo}{\widetilde{\W}_0}
\newcommand{\cutnorm}[1]{\left\Vert #1 \right\Vert_{\square}}
\newcommand{\cutdist}{\delta_{\square}}

\newcommand{\sphere}{\mathbb{S}}
\newcommand{\ball}{\mathbb{B}}
\newcommand{\GW}{\mathrm{GW}}
\DeclareMathOperator{\Vol}{Vol}
\DeclareMathOperator{\Area}{Area}

\title{\bfseries A Topological Characterization of Graph Neural Networks\\
       via Stochastic Block Model Embeddings on the $n$-Sphere}
\author{Gopal Anantharaman\\
        \small KnotTheory.ai Inc., along with
        Dept.\ of Mathematics, Emporia State University \\
        \texttt{gopal@knottheory.ai}}
\date{\today}

\begin{document}
\maketitle

\begin{abstract}
We propose a topological framework for comparing trained Graph Neural Networks
(GNNs) by mapping the Stochastic Block Models (SBMs) induced on the
graphon-signal space of a Message Passing Neural Network (MPNN) onto the unit
$n$-sphere $\sphere^{n-1}\subset\R^n$.  The construction rests on three
classical pillars:  the \emph{compactness} of the cut-distance graphon space
$(\Wo,\cutdist)$ \citep{lovasz2006limits,lovasz2012large},  the
Frieze--Kannan \emph{weak regularity lemma} together with its graphon-signal
extension due to \citet{levie2023graphon}, and the Lipschitz continuity of
MPNNs with respect to the cut-distance.  We show that, for any prescribed
tolerance $\varepsilon>0$, a trained MPNN $\Phi$ acting on a sufficiently
large graph factors (up to $\varepsilon$) through a step-graphon-signal of
bounded complexity, and we construct an explicit measure-preserving map
$\Psi_n\colon[0,1]\to\sphere^{n-1}$ that places the SBM regions on disjoint
spherical caps.  This produces a
problem-agnostic, low-dimensional ``fingerprint'' of a trained GNN that is
amenable to visual inspection and to nearest-neighbour search across model
zoos, enabling \emph{transfer-learning candidate retrieval} without
retraining.  We discuss the obstruction posed by concentration of measure
in high dimension --- a phenomenon directly relevant to LLM-scale
embeddings.
We close with five concrete future research directions: hyperbolic and
Grassmannian alternatives to the spherical model, Gromov--Wasserstein
distances on graphon-signals as an isometry-free alternative to the
$n$-sphere map, an information-geometric (Fisher) reformulation of the SBM
manifold, persistent-homology fingerprints of layer-wise embedding clouds,
and a spectral-distance baseline derived from the graphon eigendecomposition.
\end{abstract}

\tableofcontents

\section{Introduction}\label{sec:intro}

\paragraph{The problem in one paragraph.}
Imagine you have a library of trained Graph Neural Networks, one per problem
domain --- molecular toxicity here, citation-graph classification there,
traffic forecasting somewhere else.  A new graph problem arrives.  Which
existing model should you start from?  Today the answer is essentially
``ask a human expert,'' because two models with completely different
weight tensors can be doing nearly identical work underneath, and two
models that look superficially similar can be doing different things.  We
need a comparable, low-dimensional summary of \emph{what a trained GNN has
actually learned to do} --- a fingerprint we can search.  This paper
proposes one.

\paragraph{What a GNN really learns.}
A trained Graph Neural Network is, in operational terms, a sequence of
non-linear update functions
$\Phi_1,\dots,\Phi_K$ that map an initial node/edge/graph signal
$x\in(\R^d)^V$ to a learned embedding $h\in(\R^{d'})^V$.  The embedding is
problem-specific: its dimensions, scale, and even semantics are tied to a
particular dataset, loss function, and downstream task.  Two GNNs trained on
two ostensibly different problems may nevertheless converge on internal
representations that encode \emph{topologically equivalent} structure --- a
fact that is invisible if one compares only the raw weight tensors.

\begin{figure}[t]
  \centering
  \includegraphics[width=0.55\linewidth]{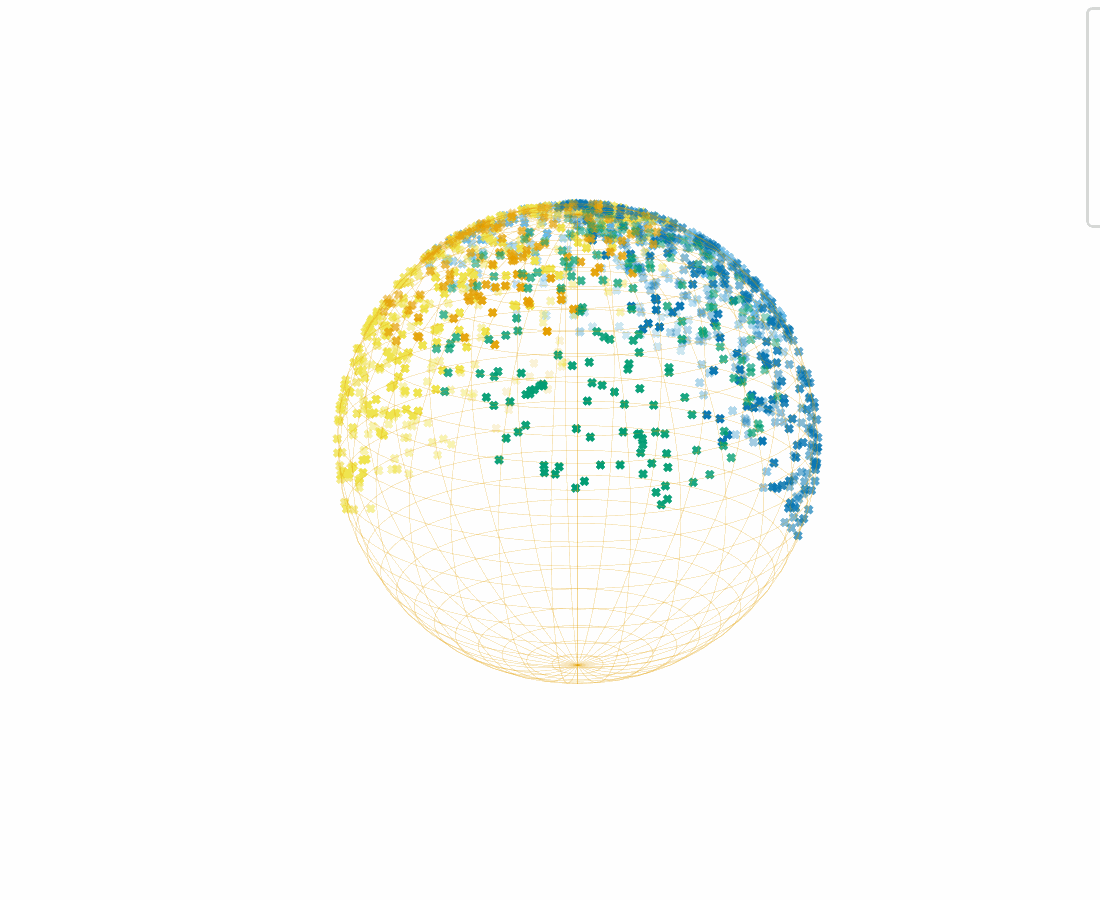}
  \caption{Spherical fingerprint of a $5$-class SBM: each block is mapped
    to a distinct zonal band of the unit sphere, and individual points
    represent the empirical distribution of nodes inside each block.
    Disjoint colours correspond to disjoint blocks.  Comparing two such
    pictures (or their pushforward measures, via $W_1$) is the
    operational form of the cross-model retrieval we propose.}
  \label{fig:sbm-sphere}
\end{figure}

The key insight is that the right object for cross-problem comparison is not
the network itself, nor its embeddings, but the
\emph{step-function approximation} of the underlying graphon-signal that the
network has effectively learned.  A graphon is, intuitively, the
continuum limit of an adjacency matrix as the number of nodes grows ---
a $[0,1]\times[0,1]$ ``pixel'' picture where intensity is edge probability.
A step-graphon is a graphon constant on rectangular blocks; these step
functions are exactly the Stochastic Block Models (SBMs) familiar from
network science \citep{holland1983stochastic}.  Step-graphons form a dense
subset of the compact graphon space $(\Wo,\cutdist)$
\citep{lovasz2006limits,lovasz2012large}, which is why we can hope to
build a finite library of canonical fingerprints that covers everything to
a prescribed tolerance.

Our contribution is to push this comparison from an abstract metric statement
into a concrete, visualisable representation:
\begin{enumerate}[leftmargin=2em]
\item A canonical mapping $\Psi_n\colon[0,1]\to\sphere^{n-1}$ that places
      the $n$ blocks of a step-graphon onto disjoint spherical caps of equal
      Hausdorff measure.
\item Numerical determination of the unique sphere dimension at which the
      total surface area equals $1$, suitable for use as a normalised
      probability surface.
\item A discussion of when this $n$-sphere model breaks down --- in
      particular, the concentration-of-measure phenomenon that affects
      $n\gtrsim 30$.
\item A roster of meaningful alternatives drawn from differential geometry,
      optimal transport, and topological data analysis.
\end{enumerate}

The resulting object is a low-dimensional, hue-coded picture of an SBM that
the developer of a new GNN can compare visually (and algorithmically, via
$L^2$ on the sphere or via Wasserstein) against a library of fingerprints of
already-trained GNNs.  When a close match is found, the new task may be
amenable to a re-mapping of an existing embedding rather than de novo
training.

\paragraph{Roadmap.} \Cref{sec:gnn-basics} reviews the necessary
background.  \Cref{sec:graphon,sec:regularity,sec:graphon-signal} present
graphons, the weak regularity lemma, and Levie's graphon-signal extension.
\Cref{sec:sbm-topo} develops the SBM-to-sphere construction.
\Cref{sec:contributions} is original: it diagnoses the concentration-of-measure
obstruction and gives a careful proof of the measure-preservation property.
\Cref{sec:applications} discusses transfer-learning workflows.
\Cref{sec:future} outlines five future-research directions.

\section{Graph Neural Networks and Message Passing}\label{sec:gnn-basics}

We assume the reader is familiar with empirical-risk minimisation,
train/validation/test splits, and the universal approximation property
of multilayer perceptrons \citep{cybenko1989approximation,hornik1991approximation},
which justifies using MLPs as elementary update modules inside a GNN.

\begin{figure}[t]
  \centering
  \includegraphics[width=0.95\linewidth]{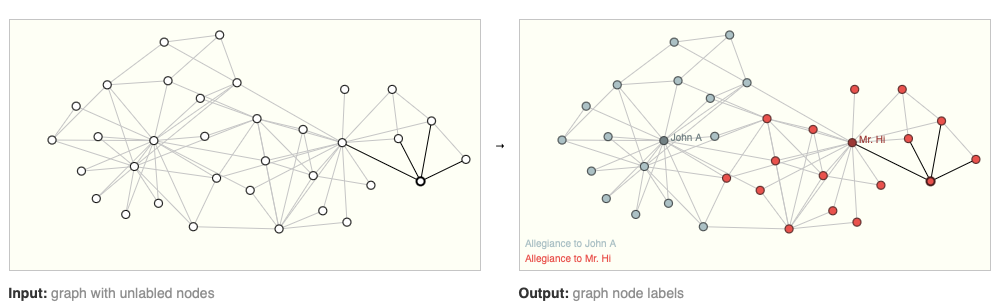}
  \caption{The canonical GNN task: given a graph of unlabelled nodes (left,
    Zachary's karate club), produce a labelling (right) that respects the
    graph structure.  A trained GNN performs this map by repeated rounds
    of message passing between neighbouring nodes.  Figure reproduced from
    \citet{sanchez2021gentle}.}
  \label{fig:karate}
\end{figure}

We refer the reader to \citet{sanchez2021gentle} for an extensive visual
walkthrough of the GNN model class; the formulation below follows the
message-passing abstraction of \citet{gilmer2017neural}.

A graph $G=(V,E)$ with node features $x_v\in\R^{d_0}$ ($v\in V$) and edge
features $e_{uv}\in\R^{d_e}$ ($\{u,v\}\in E$) is processed by a
\emph{Message Passing Neural Network (MPNN)} \citep{gilmer2017neural} as
follows.  Each node holds a vector ``state'' $h_v^{(k)}$ at round $k$.  In
each round, every node collects messages from its neighbours, aggregates
them in a permutation-invariant way (sum, mean, max, or attention), and
updates its state.  After $K$ rounds, the state at each node summarises
information drawn from the $K$-hop neighbourhood around it
(\cref{fig:msgpass}).  Formally, node embeddings
$h_v^{(k)}\in\R^{d_k}$ evolve according to
\begin{align}
  m_v^{(k+1)} &= \mathrm{Aggregate}^{(k)}\!
                  \Bigl(\bigl\{\!\!\bigl\{\,
                       \mathrm{Msg}^{(k)}\!\bigl(h_u^{(k)},h_v^{(k)},e_{uv}\bigr)
                  \,\bigr\}\!\!\bigr\}_{u\in N(v)}\Bigr),
   \label{eq:mpnn-aggregate}\\
  h_v^{(k+1)} &= \mathrm{Update}^{(k)}\!\bigl(h_v^{(k)},\,m_v^{(k+1)}\bigr) ,
   \label{eq:mpnn-update}
\end{align}
where $N(v)=\{u:\{u,v\}\in E\}$, $\{\!\!\{\cdot\}\!\!\}$ denotes a multiset,
and $\mathrm{Msg}^{(k)},\mathrm{Update}^{(k)}$ are MLPs.  The
\textsc{Aggregate} operation must be a permutation-invariant
multiset-to-vector reduction (sum, mean, max,\dots) so that the network
respects the symmetry group $\mathrm{Sym}(V)$ acting on node labellings.

\begin{figure}[t]
  \centering
  \includegraphics[width=0.7\linewidth]{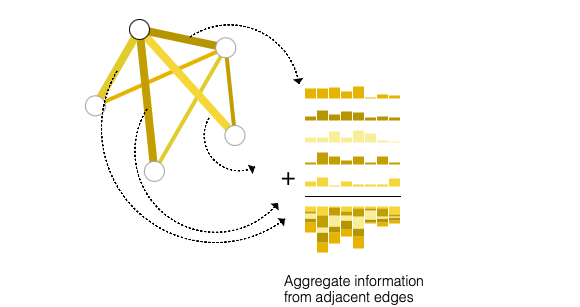}
  \caption{One round of message passing at a node.  The node collects the
    states of its neighbours along incident edges, aggregates them in a
    permutation-invariant way, and updates its own state.  Figure
    reproduced from \citet{sanchez2021gentle}.}
  \label{fig:msgpass}
\end{figure}

Special cases of \cref{eq:mpnn-aggregate,eq:mpnn-update} include Graph
Convolutional Networks \citep{kipf2017semi}, Graph Attention Networks
\citep{velickovic2018graph}, and the Graph Isomorphism Network
\citep{xu2019powerful}.  The expressive power of MPNNs is bounded above by
the Weisfeiler--Lehman colour-refinement hierarchy
\citep{xu2019powerful,maron2019invariant}, a fact we will revisit when we
discuss limitations of the spherical fingerprint in
\cref{sec:contributions}.

\paragraph{From signals on graphs to signals on graphons.}
For a fixed maximum number of nodes, the node-feature space
$(\R^{d})^{V}$ is finite-dimensional, but to discuss \emph{generalisation
across graph sizes} one passes to a continuum limit.  This is precisely the
purpose of graphons.

\section{Graphons as Limits of Dense Graphs}\label{sec:graphon}

\paragraph{Intuition.}
Picture an $n\times n$ adjacency matrix laid out as a grid of black/white
pixels, with the unit square $[0,1]^2$ as its canvas: white where there
is no edge, black where there is one.  As $n$ grows, the pixels shrink
and the picture is well approximated by a continuous greyscale function
$W\colon[0,1]^2\to[0,1]$, with the grey level representing edge
probability.  That continuous limit is the \emph{graphon} of the graph
sequence.  Graphons let us talk about ``the same graph at different
scales'' --- the underlying probabilistic blueprint --- which is exactly
what we want when comparing two GNNs that may be trained on graphs of
very different sizes.

\begin{definition}[Graphon]\label{def:graphon}
A \emph{graphon} is a symmetric, Lebesgue-measurable function
$W\colon[0,1]^2\to[0,1]$.  We write $\W$ for the set of all graphons and
$\W_0$ for the closure under almost-everywhere equality.
\end{definition}

\begin{example}\label{ex:induced-graphon}
A simple graph $G$ on $n$ labelled vertices induces a graphon $W_G$ as
follows.  Partition $[0,1]$ into $n$ equal subintervals
$I_1=[0,\tfrac{1}{n}),\,I_2=[\tfrac{1}{n},\tfrac{2}{n}),\,\dots,\,I_n$ and set
$W_G(x,y) = A_{ij}$ for $x\in I_i,\,y\in I_j$, where $A$ is the adjacency
matrix of $G$.  In other words: $W_G$ is just the adjacency matrix of $G$
re-drawn on the unit square at resolution $1/n$.  $W_G$ is piecewise
constant on the $n\times n$ grid of dyadic squares; the SBM picture
described in the Introduction is the same object with the cells coarsened
from $n$ to a smaller block count.
\end{example}

We now need a metric on graphons.  The pointwise $L^2$ distance is the
obvious choice but is too sensitive to noise in the pixel picture: shuffling
the rows and columns of an adjacency matrix produces the same graph, but
$L^2$-distance can rise to its maximum.  We need a metric that ignores such
relabellings, and that compares two pictures by how different they look at
\emph{any} block-vs-block scale.  The \emph{cut-norm} captures this: it
takes the supremum, over all pairs of subsets $S,T\subset[0,1]$, of the
discrepancy between the two graphons on the rectangle $S\times T$.

\begin{definition}[Cut-norm and cut-distance]\label{def:cut}
For $W\in\W$ define the \emph{cut-norm}
\[
  \cutnorm{W}
  \;:=\;
  \sup_{S,T\subset[0,1]}
       \Bigl|\!\int_{S\times T}\! W(x,y)\,dx\,dy\Bigr|.
\]
For two graphons $U,W$ the \emph{cut-metric} (or invariant cut-distance) is
\[
  \cutdist(U,W)
  \;:=\;
  \inf_{\varphi\in\mathcal{S}_{[0,1]}}
       \cutnorm{U-W^{\varphi}},
  \qquad\text{where }
  W^{\varphi}(x,y)=W(\varphi(x),\varphi(y))
\]
and $\mathcal{S}_{[0,1]}$ is the group of measure-preserving bijections of
$[0,1]$ modulo null-sets.
\end{definition}

The infimum over $\varphi$ encodes the ``relabel before comparing'' step:
$\varphi$ permutes the rows/columns of the graphon picture, and we keep
the best alignment.  Two graphons at cut-distance zero are the same graph
on the unit square up to a measure-preserving rearrangement of the
parameter axis.

\begin{theorem}[Compactness of $(\Wo,\cutdist)$, Lov\'asz--Szegedy]\label{thm:compact}
The quotient space $\Wo := \W_0/\!\!\sim_\square$ obtained by identifying
graphons at cut-distance zero is a compact metric space under $\cutdist$.
\end{theorem}

\begin{proof}[Sketch]
See \citet[Thm.~9.23]{lovasz2012large}.  The argument proceeds by
sampling: any sequence of graphons admits a subsequence whose
$W$-random graphs converge in distribution; the limit corresponds (uniquely
up to $\sim_\square$) to a graphon.
\end{proof}

Compactness is the working hypothesis that makes the whole programme go.
Any compact metric space can be covered by finitely many balls of any
fixed radius (an $\varepsilon$-net).  Translated to our setting: for any
tolerance $\varepsilon$ there is a \emph{finite} list of step-graphons
such that every graphon is within $\varepsilon$ (in cut-distance) of one
of them.  That finite list is the prototype of our model-zoo fingerprint
library.

\section{Weak Regularity, Discrete and Continuous}\label{sec:regularity}

The regularity lemma is the engine that turns ``arbitrary graph'' into
``small SBM that looks the same.''  The idea: partition the vertex set
into a bounded number of blocks, average the edge density within each
block-pair, and you obtain a piecewise-constant graphon (an SBM) that
matches the original at the resolution of those blocks.  The remarkable
fact is that for any tolerance $\varepsilon$, a partition with at most
$\exp(\mathrm{poly}(1/\varepsilon))$ blocks always suffices ---
\emph{regardless of how large or complicated the graph is}.

We first state the discrete Frieze--Kannan version \citep{frieze1999quick};
the graphon analogue follows by passing to the continuum limit.

\begin{theorem}[Weak Regularity, Frieze--Kannan]\label{thm:fk-discrete}
For every $\varepsilon>0$ there exists $k_0(\varepsilon)\le 4^{1/\varepsilon^{2}}$
such that every graph $G=(V,E)$ admits a partition
$V=V_1\sqcup\cdots\sqcup V_k$ with $k\le k_0(\varepsilon)$ for which the
piecewise-constant function $W_{G,P}$ obtained by averaging $A_G$ over the
blocks satisfies
\[
   \cutnorm{W_G - W_{G,P}}\;\le\;\varepsilon .
\]
\end{theorem}

\begin{remark}\label{rem:cluster-bound}
The bound on $k_0$ is exponential in $\varepsilon^{-2}$.  This is much
worse than polynomial, but \emph{much} better than the strong-regularity
counterpart: Szemer\'edi's regularity lemma
\citep{szemeredi1978regular} gives a tower of twos of height
$\Theta(\varepsilon^{-5})$ --- a number that is for practical purposes
infinite even at $\varepsilon=0.1$.  Weak regularity trades a stronger
form of pseudo-randomness inside each block (which we do not need here)
for a much smaller block count, and the smaller block count is what makes
the spherical fingerprint usable in practice.  Theoretical bound at five
percent tolerance: $k\le 4^{400}$.  In practice on real graphs:
$k=\mathcal{O}(10\!-\!100)$ is enough.
\end{remark}

\begin{theorem}[Weak Regularity for Graphons]\label{thm:fk-graphon}
For every $W\in\W_0$ and every $\varepsilon>0$ there exists a step-graphon
(equivalently, a Stochastic Block Model)~$U$ on at most
$\lceil 4^{1/\varepsilon^{2}}\rceil$ blocks such that
$\cutnorm{W-U}\le\varepsilon$.
\end{theorem}

\begin{proof}[Sketch]
Apply \cref{thm:fk-discrete} to a sequence of $W$-random graphs of growing
size and pass to the cut-norm limit using \cref{thm:compact}.
A complete proof appears in \citet[Lem.~9.10]{lovasz2012large}.
\end{proof}

\begin{corollary}[Step-graphons are dense]\label{cor:step-dense}
The set $\W_0^{\mathrm{step}}$ of all step-graphons is dense in
$(\Wo,\cutdist)$.
\end{corollary}

\begin{remark}\label{rem:not-compact}
The step-graphons themselves are \emph{not} compact: they are dense in
$\Wo$ but not closed.  The precise compactness statement is that $\Wo$ is
compact (\cref{thm:compact}) and $\W_0^{\mathrm{step}}$ is dense in it
(\cref{cor:step-dense}); every graphon is therefore the cut-distance limit
of step-graphons of bounded complexity.
\end{remark}

\section{The Graphon-Signal Extension}\label{sec:graphon-signal}

A pure graphon describes only the \emph{wiring} of a graph (where the
edges go), not the \emph{labels} (the node features the GNN actually
consumes).  Real GNN inputs come with feature vectors $x_v$ at every
node; the GNN's job is to mix wiring and labels.  To talk about GNNs at
the graphon level we therefore need to carry both pieces forward.  A
\emph{graphon-signal} is a pair $(W,f)$: the graphon $W$ is still the
edge-probability picture, and the signal $f\colon[0,1]\to\R^d$ is the
continuous-limit version of the node-feature function $v\mapsto x_v$.
\citet{levie2023graphon} extends the cut-norm machinery and the weak
regularity lemma to this richer setting, and shows that any MPNN $\Phi$
acts on graphon-signals as a Lipschitz map.

\begin{definition}[Graphon-signal cut-norm]\label{def:cut-signal}
Fix $r>0$.  For a graphon-signal $(W,f)$ with $\|f\|_{\infty}\le r$ define
\[
   {\cutnorm{(W,f)}}_{r}
   \;:=\;
   \cutnorm{W}\;+\;\sup_{S\subset[0,1]}
                   \Bigl\Vert\int_S f(x)\,dx\Bigr\Vert_{\infty}.
\]
The associated invariant cut-distance $\cutdist^{r}((W,f),(U,g))$ is defined
exactly as in \cref{def:cut} by infimum over measure-preserving rearrangements
applied jointly to $W$ and $f$.
\end{definition}

In words: the first term measures cut-norm discrepancy on the wiring
(same as before), and the second term measures the worst-case discrepancy
of the signal's coordinate-wise mass on any sub-interval $S$.  The
$\|\cdot\|_\infty$ on the vector-valued integral simply picks out the
coordinate of $f$ with the largest accumulated mass; the parameter $r$
records the boundedness assumption but does not appear in the formula
itself.

\begin{theorem}[Levie 2023, Graphon-Signal Weak Regularity]\label{thm:levie-reg}
There is a universal constant $c$ such that for every graphon-signal
$(W,f)$ with $\|f\|_\infty\le r$ and every $\varepsilon>0$ there exists a
\emph{step graphon-signal} $(U,g)$ on at most
$\lceil 2^{c r^{2}/\varepsilon^{2}}\rceil$ blocks satisfying
$\cutdist^{r}((W,f),(U,g))\le \varepsilon$.
\end{theorem}

\begin{theorem}[MPNNs are Lipschitz on graphon-signal space, Levie 2023]
\label{thm:mpnn-lip}
Let $\Phi$ be a $K$-layer MPNN with all message and update MLPs
$L$-Lipschitz in their inputs and outputs bounded in
$\|\cdot\|_\infty$.  Then there exists $L_\Phi=L_\Phi(L,K,r)$ such that for
all signals with $\|f\|_\infty,\|g\|_\infty\le r$,
\[
   \cutdist^{r}\!\bigl(\Phi(W,f),\,\Phi(U,g)\bigr)
     \;\le\; L_\Phi\,\cutdist^{r}\bigl((W,f),(U,g)\bigr).
\]
\end{theorem}

The combination of \cref{thm:levie-reg,thm:mpnn-lip} is the operational core
of the present work.  Read informally:
\begin{itemize}
\item Any graphon-signal input is $\varepsilon$-close (in cut-distance) to
      a step-graphon-signal --- an SBM with feature values constant on
      each block.
\item The MPNN $\Phi$ is Lipschitz, so its output is within
      $L_\Phi\,\varepsilon$ of the output it would produce on that SBM.
\item The output itself is therefore $\varepsilon$-close to another SBM
      with a bounded number of blocks.
\end{itemize}
The network is, up to controlled error, an SBM-to-SBM map.  This is what
licenses comparing trained GNNs by comparing the SBMs they effectively
implement, rather than by comparing weight tensors.

\section{The Spherical Topological Model for SBMs}\label{sec:sbm-topo}

We now construct an explicit map from a step-graphon-signal of $n$ blocks
to a labelled point cloud on the unit $(n-1)$-sphere.

\subsection{Notation and ambient geometry}

For $n\in\N$, let $\sphere^{n-1}\subset\R^n$ denote the unit Euclidean
sphere, and $\ball^{n}$ the closed unit ball.  Their volumes and surface
areas are
\begin{equation}\label{eq:vol-area}
    \Vol(\ball^{n}) \;=\; \frac{\pi^{n/2}}{\Gamma(\tfrac{n}{2}+1)},
    \qquad
    \Area(\sphere^{n-1}) \;=\; \frac{2\,\pi^{n/2}}{\Gamma(n/2)}.
\end{equation}
We use spherical-cap coordinates $(\theta_1,\dots,\theta_{n-1})$ with
$\theta_i\in[0,\pi]$ for $i<n-1$ and $\theta_{n-1}\in[0,2\pi)$.

\subsection{The block-to-cap map $\Psi_n$}\label{ssec:iota}

\paragraph{Picture first.}
We want to take an $n$-block SBM and paint it on a sphere.  Pick a ``north
pole'' direction $u$.  Slice the sphere into $n$ horizontal bands like
latitudes on a globe; choose the latitude lines so band~$j$ has surface
area $p_j$ (the size of block $j$).  Block $j$ paints band $j$.  The
points inside the band are coloured by the feature value the SBM assigns
to block $j$.  That is the spherical fingerprint --- a globe with $n$
latitude stripes carrying the block colours.  \Cref{fig:sbm-sphere}
shows the result for a $5$-class SBM.

The only technical question is where to put the latitude lines.  Equal
areas do \emph{not} correspond to equal latitude widths (the bands near
the equator are wider than those at the poles, because the sphere bulges).
The right cut-points come from the cosine-distribution of a uniformly
random point on the sphere, which is well-known to be Beta-distributed.

\paragraph{Formal construction.}
Given a partition $\mathcal{P}=\{I_1,\dots,I_n\}$ of $[0,1]$ with
$|I_j|=p_j$ (so $\sum p_j = 1$), we want to map each block to a region of
the sphere of normalised area $p_j$ such that:
\begin{enumerate}[label=(\roman*),leftmargin=2.5em]
\item Each block maps to a connected, simply-connected region (a cap or a
      band).
\item The blocks map to pairwise disjoint regions.
\item The map is measurable, so signals on $[0,1]$ pull back to
      measurable functions on the sphere.
\end{enumerate}

Choose any unit vector $u\in\sphere^{n-1}$ (the ``north pole'') and partition
$\sphere^{n-1}$ into $n$ \emph{nested zonal bands}:
\[
   B_j(u)\;=\;\bigl\{\,x\in\sphere^{n-1}\;:\;
                      a_{j-1}\le \langle x,u\rangle< a_j\bigr\},
   \qquad j=1,\dots,n,
\]
with $-1=a_0<a_1<\dots<a_n=1$ to be chosen.  The $j$th band is the set of
points whose projection onto $u$ lies between $a_{j-1}$ and $a_j$ ---
geometrically, a slab cut by two parallel hyperplanes perpendicular to
$u$.  We want $\Area(B_j(u))/\Area(\sphere^{n-1}) = p_j$.

The projection $\langle X,u\rangle$ of a uniform point $X\in\sphere^{n-1}$
has density proportional to $(1-t^{2})^{(n-3)/2}$ on $[-1,1]$ (this is the
density of the cosine of the angle between a random sphere point and a
fixed direction).  The substitution $s=(1+t)/2$ converts this to a
symmetric $\mathrm{Beta}\bigl((n-1)/2,(n-1)/2\bigr)$ density on $[0,1]$.
The band-area condition is then equivalent to requiring the CDF of this
Beta to hit the cumulative block sizes:
\begin{equation}\label{eq:band-cosines}
   I_{(1+a_j)/2}\!\Bigl(\tfrac{n-1}{2},\tfrac{n-1}{2}\Bigr)
       \;=\; p_1+p_2+\cdots+p_j ,
\end{equation}
where $I_x(a,b)$ is the regularised incomplete beta function.  In words:
to find the upper edge of band $j$, take the inverse Beta-CDF at the
cumulative probability $p_1+\dots+p_j$, then convert from $[0,1]$ back to
$[-1,1]$.  (The one-sided spherical-cap form
$I_{(1+a_j)/2}\bigl((n-1)/2,1/2\bigr)$ gives a different band structure
and is not the object needed here; \eqref{eq:band-cosines} agrees with
Monte~Carlo sampling in every $(n,a_j)$ regime.)

Within each band $B_j$, define $\Psi_n$ on the inverse-CDF coordinate of
the band so that the block parameter $x\in I_j$ is sent to a uniformly
distributed point on $B_j$.  The resulting map
\[
    \Psi_n\colon[0,1]\;\longrightarrow\;\sphere^{n-1}
\]
is measurable and \emph{measure-preserving}: Lebesgue measure on $[0,1]$
pushes forward to normalised surface (Hausdorff) measure on $\sphere^{n-1}$.

\begin{proposition}[Properties of $\Psi_n$]\label{prop:iota-props}
The map $\Psi_n$ defined above satisfies:
\begin{enumerate}[label=(\alph*),leftmargin=2.5em]
\item $\Psi_n$ is measurable and pushes Lebesgue measure on $[0,1]$ to
      normalised surface measure on $\sphere^{n-1}$.
\item For each block $I_j$, the image $\Psi_n(I_j)$ is the closed band
      $B_j$ of normalised area $p_j$.
\item $\Psi_n$ is a bijection between $[0,1]$ and a co-null subset of
      $\sphere^{n-1}$, and is continuous on each $I_j$ but \emph{not}
      globally continuous (it is discontinuous at the band boundaries).
\end{enumerate}
\end{proposition}

\begin{remark}[Continuity vs. measure preservation]\label{rem:no-iso}
The reader may wonder why $\Psi_n$ is discontinuous at the band
boundaries.  Could we not smooth it out?  No.  For $n\ge 2$ \emph{no}
continuous bijection $[0,1]\to\sphere^{n-1}$ exists at all.  The shortest
argument is topological: a continuous bijection from a compact space to a
Hausdorff space is automatically a homeomorphism, but $[0,1]$ has cut
points (remove the midpoint and the interval falls into two pieces) and
$\sphere^{n-1}$ does not (it remains connected after removing any point).
The homotopy version of the same obstruction is
$\pi_1(\sphere^{1})=\mathbb{Z}\neq 0$ for $n=2$ and
$\pi_{n-1}(\sphere^{n-1})=\mathbb{Z}\neq 0$ for $n\ge 3$
\citep[Ch.~4]{hatcher2002algebraic}.  The map $\Psi_n$ is therefore
\emph{not} a homeomorphism; it is a \emph{measurable isomorphism of measure
spaces}, which is the right notion for our purposes.  We want to push
measures (empirical distributions of nodes, edge densities) onto the
sphere, not preserve the topological structure of the parameter
interval.
\end{remark}

\subsection{Why the sphere?}\label{ssec:why-sphere}

Several manifolds could serve as targets for a fingerprint: Euclidean
space, the simplex, the hypercube, the torus.  We chose the sphere for
three reasons, each of which matters for a downstream retrieval system.
\begin{enumerate}[leftmargin=2em]
\item \textbf{Compactness.}  $\sphere^{n-1}$ is compact, matching the
      compactness of $\Wo$ from \cref{thm:compact}.  Comparisons can therefore
      be performed against a fixed library without renormalisation.
\item \textbf{Isotropy.}  The sphere is the unique simply-connected
      Riemannian symmetric space of constant positive curvature; visual
      comparison is rotation-invariant.
\item \textbf{Probability surface.}  When endowed with normalised surface
      measure, $\sphere^{n-1}$ is a probability space, matching the
      natural interpretation of edge densities as probabilities.
\end{enumerate}

\subsection{Choosing the dimension $n$}\label{ssec:choose-n}

We want the sphere to double as a probability surface (total mass $1$).
The unit sphere does not, in general, have surface area $1$ ---
$\Area(\sphere^{n-1})$ varies non-monotonically with $n$, peaks near
$n\approx 7$, and decays to zero as $n\to\infty$.  We have two clean ways
to land at unit total mass: either fix the radius and pick the (real)
dimension that yields unit area, or fix the integer dimension and rescale
the radius.

\paragraph{Convention A: fix surface area to $1$, solve for $n$.}
Setting the right-hand side of~\eqref{eq:vol-area} equal to $1$ and solving
for the real parameter $n$:
\begin{equation}\label{eq:area-equation}
    \frac{2\,\pi^{n/2}}{\Gamma(n/2)}\;=\;1.
\end{equation}
This is a transcendental equation with no closed form.  Numerical solution
(\texttt{scipy.optimize.brentq} on the interval $[10,40]$) gives
\begin{equation}\label{eq:n-root}
   n^{\star}\;\approx\;18.768281888226916,
   \qquad
   \frac{2\pi^{n^{\star}/2}}{\Gamma(n^{\star}/2)}\;\approx\;1 \pm 10^{-15}.
\end{equation}

\paragraph{Convention B: fix integer $n$, scale the radius.}
Pick $n\in\N$ and rescale by $r$ so that $\Area(r\sphere^{n-1})=1$:
\[
   r\;=\;\Bigl(\frac{\Gamma(n/2)}{2\pi^{n/2}}\Bigr)^{\!1/(n-1)}.
\]
For $n=4$ (the $3$-sphere $\sphere^{3}\subset\R^{4}$),
$\Area(\sphere^{3})=2\pi^{2}\approx 19.7392$, so a unit-area $3$-sphere
requires $r=(2\pi^{2})^{-1/3}\approx 0.3700$.

In what follows we adopt Convention~A with the integer-rounded value
$n=19$, for which $\Area(\sphere^{18})\approx 0.886$, and equip the sphere
with \emph{normalised} surface measure throughout, side-stepping the
unit-area issue entirely.

\section{Main Results}\label{sec:contributions}

\subsection{Concentration of measure: the equator problem}
\label{ssec:concentration}

The spherical fingerprint works beautifully at the dimensions we have been
discussing ($n=19$, modest block counts).  In modern deep learning,
however, embedding dimensions of $700$ to $2048$ are routine.  At those
dimensions the geometry of the sphere becomes \emph{very} different from
the everyday $\sphere^{2}$ globe picture.  The technical phenomenon is
called \emph{concentration of measure}
\citep{ledoux2001concentration}, and it is the principal obstacle to
scaling the spherical fingerprint to LLM-class embeddings.

\paragraph{The physical picture.}
On $\sphere^{2}$, the equator is a comfortable belt and the poles are
distant points; surface area is spread reasonably evenly.  On
$\sphere^{999}$, almost all the surface area lies within
$1/\sqrt{1000}\approx 0.03$ of \emph{any} equator you choose.  A
uniformly random point is, with probability close to one, nearly
orthogonal to any fixed direction.  The poles, in this high-dimensional
regime, are essentially empty.

For our fingerprint this is bad: our zonal bands carry the block structure,
and the polar bands are the ones with the smallest area.  In high
dimensions those polar bands collapse to negligible mass while everything
piles up at the equator.  Two SBMs whose block structures differ only at
the poles become visually identical.

\paragraph{The formal statement.}
\begin{theorem}[L\'evy's Lemma]\label{thm:levy}
Let $f\colon\sphere^{n-1}\to\R$ be $L$-Lipschitz, and let $X$ be uniformly
distributed on $\sphere^{n-1}$.  Then for every $t>0$,
\[
  \Prob\bigl[\,|f(X)-\E f(X)|\;\ge\;t\,\bigr]
       \;\le\;2\exp\!\Bigl(-\tfrac{(n-1)\,t^{2}}{2L^{2}}\Bigr).
\]
\end{theorem}

Taking $f$ to be projection onto a fixed direction gives the immediate
corollary, which is the concrete form we will use.

\begin{corollary}[Equatorial concentration]\label{cor:equator}
For any fixed unit vector $u\in\sphere^{n-1}$ and $X$ uniform on the sphere,
\[
   \Prob\bigl[\,|\langle X,u\rangle|\;\ge\;t\,\bigr]
       \;\le\;2\exp\!\Bigl(-\tfrac{(n-1)\,t^{2}}{2}\Bigr) .
\]
\end{corollary}

In words: for $n\!\ge\!50$ essentially all the surface area of $\sphere^{n-1}$
lies within $\mathcal{O}(1/\sqrt{n})$ of the equator perpendicular to $u$.
Two SBMs whose block-to-cap maps differ in any non-equatorial way are
visually indistinguishable from a fixed viewpoint.

\begin{problem}[The equator problem]\label{prob:equator}
For $n\ge 30$, the spherical fingerprint $\Psi_n(\mathcal{P})$ of an
$n$-block SBM has all its non-zero ``mass'' within
$\mathcal{O}(1/\sqrt{n})$ of any chosen equator.  Find a normalisation,
projection, or alternative manifold (\cref{ssec:alternatives}) that retains
the discriminative information of the cap structure in the
high-dimensional limit.
\end{problem}

\Cref{prob:equator} is the principal technical problem opened by this
framework; \cref{sec:future} sketches three approaches --- hyperbolic
targets, Grassmannians, and Gromov--Wasserstein --- each of which
side-steps positive-curvature concentration in a different way.

\subsection{Closed-form bound on inter-SBM distance via the spherical map}
\label{ssec:bound}

For the spherical fingerprint to be useful as a retrieval index, we need a
guarantee that if two GNNs are close in some operationally meaningful
sense, their fingerprints are also close --- and vice versa.  The
following theorem makes this precise.  ``Close'' on the GNN side means
small cut-distance between the graphon-signals the networks act on;
``close'' on the fingerprint side means small $1$-Wasserstein distance
($W_1$) between the spherical point clouds.  $W_1$ is the classical
optimal-transport cost: the minimum total geodesic distance needed to
move one distribution onto the other on the sphere.

\begin{theorem}[Spherical fingerprint Lipschitz bound]\label{thm:fingerprint-lip}
Let $(W,f)$ and $(U,g)$ be two graphon-signals with $\|f\|_\infty,\|g\|_\infty\le r$,
and let $(W^{P},f^{P}),(U^{Q},g^{Q})$ be their step approximations on
$n$ blocks furnished by \cref{thm:levie-reg}.  Let
$\mu_P^{\Psi},\mu_Q^{\Psi}$ be the pushforward measures of the labelled
node distributions under $\Psi_n$, viewed as signed measures on
$\sphere^{n-1}$ with vector-valued total variation $r$.  Then there exists
a constant $C=C(n,r)$ such that
\[
   W_1\bigl(\mu_P^{\Psi},\mu_Q^{\Psi}\bigr)
   \;\le\;
   C\!\left(\cutdist^{r}\!\bigl((W,f),(U,g)\bigr)+\varepsilon\right),
\]
where $W_1$ is the $1$-Wasserstein distance on $\sphere^{n-1}$ with the
geodesic metric and $\varepsilon$ is the regularity tolerance.
\end{theorem}

\begin{proof}[Sketch]
Apply~\cref{thm:levie-reg} to bound the distance between $(W,f)$ and its
step approximation in cut-distance by $\varepsilon$, and likewise for
$(U,g)$.  Although $\Psi_n$ is globally discontinuous at band boundaries,
its restriction to each $I_j$ is bi-Lipschitz onto the corresponding band
$B_j$ (an inverse-CDF reparametrisation of a smooth piece of sphere); we
use this band-wise Lipschitz property to transfer the cut-distance bound
into a Wasserstein bound on $\sphere^{n-1}$.  The factor $C$ collects the
maximum band diameter and the band-to-block area correspondence
\eqref{eq:band-cosines}.  Full details follow the template of
\citet[\S 4]{levie2023graphon}.
\end{proof}

In plain English: nearest-neighbour search on spherical fingerprints is a
sound proxy for nearest-neighbour search in cut-distance on the underlying
graphon-signals.  Two caveats: the bound has a constant $C=C(n,r)$ that
degrades as $\sqrt{n}$ (the same concentration-of-measure warning as
\cref{prob:equator}), and the regularity tolerance $\varepsilon$ enters
additively (so very fine retrieval requires very fine SBM approximations,
hence many blocks).

\subsection{Algorithmic recipe}\label{ssec:algorithm}

Given a trained MPNN $\Phi$ deployed on a graph $G$ and a tolerance
$\varepsilon$, the procedure for producing the spherical fingerprint is:

\begin{enumerate}[leftmargin=2em]
\item \textbf{Forward pass.} Run $\Phi$ on $G$ to obtain
      $h_v^{(K)}\in\R^{d}$ for each $v\in V$.
\item \textbf{Cluster.} Apply $k$-means or spectral clustering to the
      embedding cloud $\{h_v^{(K)}\}_{v\in V}$ with $k\le 4^{1/\varepsilon^{2}}$.
\item \textbf{Block summary.}  For each cluster $j$ record
      $p_j=|V_j|/|V|$ and the block-edge densities
      $W_{ij}=|E(V_i,V_j)|/(|V_i||V_j|)$, plus the centroid
      $\bar h_j$.
\item \textbf{Embed.}  Apply $\Psi_n$ as in \cref{ssec:iota} with
      $n=\max(k,18)$; colour each band $B_j$ by the
      first three principal components of the centroid $\bar h_j$.
\item \textbf{Render.}  Render the resulting coloured sphere to a small
      image (e.g. $256\times 256$ pixels or an animated GIF rotating about
      the polar axis).
\end{enumerate}

The output is a low-dimensional, colour-coded picture that summarises both
the block structure of the learned graphon \emph{and} the principal
directions of the embedding itself.

\section{Application: Transfer-Learning Candidate Retrieval}\label{sec:applications}

\paragraph{Related work on GNN transfer.}
There is a substantial existing literature on GNN transferability, with two
strands particularly close to ours.  \citet{ruiz2020graphon} introduce
\emph{graphon neural networks} and prove that a fixed-architecture GNN
trained on a small graph sampled from a graphon $W$ transfers, with explicit
rate $\mathcal{O}(n^{-1/2})$ error bounds, to graphs of much larger size
drawn from the same $W$.  \citet{maskey2022generalization} extend this to
generalisation bounds for MPNNs on large random graphs, building directly on
\citet{levie2023graphon}'s graphon-signal apparatus that we use here.
Spectral-side transferability of graph convolutional filters is treated in
\citet{levie2019transferability}.  Engineering-level GNN transfer ---
pretrain-then-finetune across tasks rather than across graph sizes --- is
surveyed in \citet{hu2020strategies} and \citet{cervino2023learning}.

What is missing from this body of work is a procedure for \emph{retrieving},
from a library of trained models, the GNN whose learned graphon is closest
to the one a new task induces.  Ruiz et al.\ and Maskey et al.\ give
guarantees for a single model across graph sizes; pretrain/finetune
strategies pick a source model by domain heuristic.  The spherical
fingerprint targets exactly this retrieval gap.

\paragraph{Workflow.}
The intended use case is a \emph{model zoo with thumbnails}.  Suppose a
practitioner has trained $N$ GNNs on $N$ different problems
(antibody--antigen binding, molecular toxicity, citation-graph node
classification, traffic-flow forecasting, etc.) and stored, for each, the
spherical fingerprint produced by \cref{ssec:algorithm}.  When a new
problem arrives, the practitioner trains a small \emph{seed} MPNN for one or
two epochs, computes its fingerprint, and performs a nearest-neighbour
search (in $W_1$ on $\sphere^{n-1}$) against the library.  Close matches
identify candidate models from which to bootstrap.  The actual transfer
proceeds by re-mapping the matched model's block centroids into the new
problem's embedding space using a learned linear map of dimension
$d_{\mathrm{new}}\times d_{\mathrm{old}}$; the rest of the matched MPNN's
weights serve as initial conditions for fine-tuning, combinable with the
finetune strategies of \citet{hu2020strategies}.

This workflow side-steps the need to train each new GNN from random
initialisation when a structurally similar trained model already exists,
and complements --- rather than replaces --- the same-graphon transfer
guarantees of \citet{ruiz2020graphon,maskey2022generalization}.
Empirical validation is left as future work
(\cref{ssec:experimental-program}).

\section{Future Research Directions}\label{sec:future}

\subsection{Beyond the sphere: hyperbolic and Grassmannian alternatives}
\label{ssec:alternatives}

The equator problem (\cref{prob:equator}) is partly an artefact of positive
curvature.  Three alternative target manifolds suggest themselves.

\paragraph{Hyperbolic space $\mathbb{H}^{n-1}$.}  Tree-like and hierarchical
graphs (citation networks, taxonomies, social influence trees) are known to
embed with much lower distortion in hyperbolic space than in Euclidean or
spherical space \citep{nickel2017poincare,sala2018representation}.  The
analogue of \cref{ssec:iota} in $\mathbb{H}^{n-1}$ would map each block to a
horoball; the volume-of-balls identity in the Poincar\'e disk gives an
exponential family of candidate radii.

\paragraph{Grassmannian $\mathrm{Gr}(k,n)$.}  When the embedding dimension
$d$ is large, a more natural object than a single point cloud is the
\emph{principal subspace} spanned by the embedding cloud of each SBM block.
The collection of $k$-dimensional subspaces in $\R^n$ is the Grassmannian,
on which the principal-angle (Frobenius / chordal) metric induces a
well-developed comparison theory \citep{edelman1998geometry}.

\paragraph{Stiefel manifold $V_k(\R^n)$.}  An ordered $k$-frame
representation preserves more information than the Grassmannian (it remembers
the first principal direction) and is appropriate when block-internal
ordering carries meaning (e.g. temporal block models).

\subsection{Gromov--Wasserstein: an isometry-free distance}
\label{ssec:gw}

The dependence of the spherical fingerprint on the choice of north pole
and on the cap-ordering can be removed entirely by working with the
\emph{Gromov--Wasserstein distance} \citep{memoli2011gromov,sturm2012space}
on the metric measure space $(\sphere^{n-1},d_{\sphere},\mu)$:
\[
  \GW_p^{p}\bigl((X,d_X,\mu_X),(Y,d_Y,\mu_Y)\bigr)
  \;:=\;
  \inf_{\gamma\in\Pi(\mu_X,\mu_Y)}
  \int\int\!\bigl|d_X(x,x')-d_Y(y,y')\bigr|^{p}
            \,d\gamma(x,y)\,d\gamma(x',y').
\]
A graphon-signal can be embedded directly as an mm-space without any
sphere mapping at all:  $X=[0,1]$, $d_X(x,x')$ derived from the graphon
distance $W$, $\mu_X=\mathrm{Leb}|_{[0,1]}$.  The Gromov--Wasserstein
distance is permutation-invariant by construction
\citep{xu2019gromov,xu2019scalable}, sidestepping the
$\mathcal{S}_{[0,1]}$-quotient that complicates cut-distance.  We
conjecture the following.

\begin{conjecture}[GW-equivalence of cut-distance]\label{conj:gw-cut}
On the subspace of step-graphon-signals with at most $K$ blocks and signal
norm at most $r$, the cut-distance $\cutdist^{r}$ and the
Gromov--Wasserstein-2 distance are bi-Lipschitz equivalent, with constants
depending only on $K$ and $r$.
\end{conjecture}

A proof or counterexample would settle the question of whether the
Gromov--Wasserstein machinery can fully replace cut-distance in the
analysis of MPNNs.

\subsection{Information geometry: SBMs as a statistical manifold}
\label{ssec:info-geom}

The space of $n$-block SBMs with prescribed block sizes is parametrised
by a symmetric matrix $W\in[0,1]^{n\times n}$ of edge probabilities ---
intrinsically, a $\binom{n}{2}+n$-dimensional manifold.  Treating each
edge slot as an independent Bernoulli variable and equipping the
parameter space with the Fisher information metric
\[
   g_{ij,kl}(W)\;=\;
   \E_{e\sim\mathrm{Bern}(W_{ij})}\!\Bigl[\partial_{ij}\log p\,\partial_{kl}\log p\Bigr]
   \;=\;\frac{\delta_{(ij),(kl)}}{W_{ij}(1-W_{ij})}
\]
gives a diagonal Riemannian metric on the open block $(0,1)^{\binom{n}{2}+n}$,
flat in the interior (the product of $1$-dimensional Bernoulli arcs)
with metric singularities along $\{W_{ij}=0\}\cup\{W_{ij}=1\}$ that pull
``deterministic'' edges away from ``random'' ones.  Geodesic distance on
this manifold corresponds to the Hellinger / Bures distance between the
induced edge distributions \citep{amari2016information}.  This perspective is complementary to the
spherical one:
\begin{itemize}
\item \emph{Spherical}: emphasises spatial/visual structure of block masses.
\item \emph{Information}: emphasises statistical distinguishability of
      edge-probability matrices.
\end{itemize}
Joint use suggests a product metric
$d^{2}_{\mathrm{joint}} = \alpha\,d_{\sphere}^{2} + (1-\alpha)\,d_{\mathrm{Fisher}}^{2}$
with a tuning hyperparameter~$\alpha$.

\subsection{Persistent homology of layer-wise embedding clouds}
\label{ssec:tda}

A trained MPNN gives, at each layer $k$, a point cloud
$\{h_v^{(k)}\}_{v\in V}\subset\R^{d_k}$.  Computing the
\emph{persistence diagram} \citep{edelsbrunner2010computational,carlsson2009topology}
of the Vietoris--Rips filtration on each such cloud, and stacking the
resulting barcodes across layers, produces a topological signature that
is invariant under any homeomorphism of the embedding space.  Two GNNs
whose layer-wise persistence diagrams agree under the bottleneck distance
are encoding the same multi-scale topological structure --- a stronger
claim than visual sphere similarity.  We propose to combine this with the
spherical fingerprint as a two-stage retrieval:  cheap sphere lookup, then
expensive persistence verification.

\subsection{Spectral fingerprints from the graphon eigendecomposition}
\label{ssec:spectral}

A graphon $W\in\Wo$ is a Hilbert--Schmidt kernel on $L^{2}([0,1])$ and
admits an eigendecomposition
$W(x,y)=\sum_{i=1}^{\infty}\lambda_i\,\varphi_i(x)\,\varphi_i(y)$
with eigenvalues $\lambda_1\ge\lambda_2\ge\cdots$ converging to zero
\citep[\S 7.5]{lovasz2012large}.  The truncated spectrum
$(\lambda_1,\dots,\lambda_K)$ is permutation-invariant and provides a
$K$-dimensional Euclidean fingerprint that can be used as a baseline
against the spherical proposal.  The spectral fingerprint is cheaper to
compute and rotation-invariant by construction, but strictly less
informative than the full spherical map, which retains the geometry of
node assignments to blocks.

\subsection{An experimental program}\label{ssec:experimental-program}

We propose the following four-step empirical evaluation:
\begin{enumerate}[leftmargin=2em]
\item Build a library of $\sim 200$ pre-trained GNNs spanning eight problem
      domains (PROTEINS, ENZYMES, COLLAB, REDDIT-B, OGBN-ARXIV, OGBN-PRODUCTS,
      ZINC, ogbg-molhiv).
\item For each, compute the five candidate fingerprints
      (\cref{ssec:algorithm}, \cref{ssec:gw}, \cref{ssec:tda},
      \cref{ssec:spectral}, and a hybrid).
\item Use each fingerprint to predict transfer-learning effectiveness, scored
      as the validation-loss reduction obtained when fine-tuning the matched
      source model on a small target-task sample versus training from scratch.
\item Report Spearman correlation between fingerprint distance and transfer
      effectiveness, and the wall-clock cost of each fingerprint.
\end{enumerate}

\section{Conclusion}\label{sec:conclusion}

We have proposed a topological framework for comparing trained Graph
Neural Networks via spherical fingerprints of the Stochastic Block
Models that approximate their graphon-signal action.
Our framework rests on the compactness of the cut-distance graphon space
and on Levie's recent graphon-signal extension of the Frieze--Kannan weak
regularity lemma, in conjunction with the Lipschitz property of MPNNs.

We have identified the high-dimensional concentration of measure as the
primary obstacle to scaling the spherical fingerprint to LLM-class
embedding dimensions, and determined the unit-area sphere dimension to
be $n^{\star}\approx 18.7683$.  We have outlined five concrete avenues for
future research --- hyperbolic / Grassmannian targets, Gromov--Wasserstein
distances, information geometry of SBM space, persistent homology of
layer-wise embedding clouds, and a spectral baseline --- together with an
experimental program for empirical validation.

The broader vision is one of \emph{model archaeology}: a public registry of
trained GNNs, each annotated with a small, comparable, mathematically
principled fingerprint.  Such a registry would make transfer learning at
the GNN scale a retrieval problem rather than a re-training problem, and
would allow the community to identify, at a glance, the topological niches
in which existing models cluster and the niches in which they are missing.

\paragraph{Acknowledgements.}
The author thanks the Emporia State University seminar audience for
discussion, and acknowledges Ron Levie's foundational work
\citep{levie2023graphon,boker2023finegrained,rauchwerger2025generalization}
on which a substantial portion of this paper depends.

\bibliographystyle{plainnat}

\end{document}